\documentclass[11pt,a4paper]{article}
\usepackage[a4paper]{geometry}
\usepackage{graphicx}
\usepackage{amssymb}

\usepackage{lineno}

\usepackage{amsmath}
\usepackage{booktabs}
\usepackage{longtable}
\usepackage[utf8]{inputenc}
\usepackage[hyphens,spaces,obeyspaces]{url}
\usepackage{pgfgantt}
\usepackage{xcolor}
\usepackage{datetime2}
\usepackage{amsthm}
\usepackage[linesnumbered,ruled,vlined]{algorithm2e}

 \newtheorem{theorem}{Theorem}

 \title{Long term planning of military aircraft flight and maintenance operations}
 \author{Franco Peschiera, Olga Battaïa, Alain Haït, Nicolas Dupin}

\date{ISAE-SUPAERO, Université de Toulouse, France}
 
\begin{document}
	
	\maketitle
	
	\begin{abstract}
	We present the Flight and Maintenance Planning (FMP) problem in its military variant and applied to long term planning. The problem has been previously studied for short- and medium-term horizons only. We compare its similarities and differences with previous work and prove its complexity. We generate scenarios inspired by the French Air Force fleet. We formulate an exact Mixed Integer Programming (MIP) model to solve the problem in these scenarios and we analyse the performance of the solving method under these circumstances. A heuristic was built to generate fast feasible solutions, that in some cases were shown to help warm-start the model.

\end{abstract}
	


%
%



\section{Introduction}
\label{sec:intro}

Maintenance and repair industries are not the easiest to measure as part of a country’s GDP. Canada estimates its companies spent 3.3\% of GDP on repairs in 2016, more than twice as much as the country spends on research and development \cite{Economist2018}. Although the value of a good maintenance policy is usually hidden behind the uneventful and correct functioning of a system, the cost of a lack of maintenance of public infrastructures can be measured from time to time. This became all too evident recently with the collapse of the Genoa bridge in Italy. The lack of investment in maintenance was one of the main reasons for this tragedy that caused the death of 43 people in August 2018.

In the aircraft industry, maintenance is done via various types of maintenance operations, or checks. These checks vary in frequency, duration and thoroughness. Type A and B checks are scheduled on a daily or weekly basis and take up to 300 man-hours. C and D checks may take several months to complete and are scheduled every 1 and 5 years respectively. Checks A to C are usually performed by the aircraft operator. However, D checks need to be performed in a specialized facility, using specific resources.

The Flight and Maintenance Planning (FMP) problem, first presented by \cite{Feo1989}, studies how these maintenance operations are scheduled and how flight activities are assigned to a fleet of aircraft along a planning horizon. It has two main variants: civil (or commercial) and military. 

In the civil variant of the FMP (\cite{Hane1995, Clarke1996, Sriram2003}), by far the one that has received most attention, aircraft need to be routed along different destinations by assigning them legs in order to build daily trips. Checks are done during the night and are usually limited to A and B type checks. The planning horizon is over a period of several days (e.g. \cite{Sriram2003} uses 7 days). The objective is to maximize profit.

The military variant assumes that each aircraft returns to the airbase after each flight. Other differences exist in terms of objectives, check frequencies, check durations and the size of the fleets concerned. This paper is given over to the military variant.

Recently, optimizing maintenance operations on military aircraft has become a priority for many governments \cite{marlow2017optimal}. In particular, the French Air Force was interested in mathematical solutions to schedule maintenance for the Mirage 2000 fleet \cite{DefiOptiplan2018} after it was revealed that increases in maintenance costs had not been followed by improved availability of aircraft \cite{Parly2017} and preliminary work on the development of mathematical models in \cite{Chastellux2016}.

To the best of our knowledge, optimization models for the military FMP problem, first introduced by \cite{sgaslik1994planning} and \cite{Pippin1998} with respect to helicopter maintenance, have always prioritized availability of the fleet under a given set of operational and demand constraints. There are several ways to understand availability. An aircraft is considered available if it is not undergoing any maintenance operation and has enough flight hours to be assigned to a mission. Another way of considering aircraft availability is the amount of flight hours remaining before mandatory maintenance. This implies that a fleet with a small number of flight hours is not as prepared as one in which every aircraft has just completed its maintenance.

As can be seen in section \ref{sec:state}, previous work has been particularly focused on the short and medium-term planning variants of this problem. This paper studies the long-term planning of military maintenance and flight operations, i.e. D-type checks, for which special considerations need to be taken into account.

This article is structured as follows. Section \ref{sec:problem} sets out a detailed description of the problem. An analysis of the previous work on FMP is given in section \ref{sec:state}. 
Section \ref{sec:complex} analyses the complexity of the problem considered. 
Section \ref{sec:model} presents a new MIP formulation for the problem. 
Section \ref{sec:experim} introduces the numerical experiment. 
Section \ref{sec:results} discusses the results obtained. 
Lastly, section \ref{sec:conclusions} provides conclusions and pointers for further work.

\section{Problem statement}
\label{sec:problem}

The goal is to assign a number of military aircraft to a given set of already scheduled missions while scheduling maintenance operations (referred to as checks) over a time horizon.

A series of $j \in \mathcal{J}$ missions are known along a horizon divided into $t \in \mathcal{T}$ periods. Since all missions are already scheduled, we know the time periods $T_{j} \subset \mathcal{T}$ in which they will be performed. Similarly, all the missions to be performed in period $t$ are known and defined by set $J_{t} \subset \mathcal{J}$. Each mission requires a certain number $R_{j}$ of aircraft $i \in \mathcal{I}$ which it uses for a time duration set by $H_{j}$ in each assigned period. Set $I_{j} \subset \mathcal{I}$ lists the aircraft that can be assigned to each mission and set $O_{i} \subset \mathcal{J}$ consists of missions for which aircraft $i$ can be used. Whenever an aircraft is assigned to a mission $j$, it needs to be assigned for at least $MT_{j}$ consecutive periods.

Each aircraft can only be assigned to a single mission in any given period. These aircraft suffer from wear and tear and require regular checks. The need for maintenance is calculated based on two indicators.

The first one is called ``remaining calendar time'' ($rct$). It expresses the amount of time (measured in time periods) after which the aircraft cannot be used any more and has to undergo a check. This value is calculated for each aircraft $i$ and each time period $t$. Similarly, ``remaining flight time'' ($rft$) is employed to measure the amount of time (measured in flight hours) that the aircraft $i$ can be used before needing a check at any given period $t$.

Each check has a fixed duration of $M$ periods and cannot be interrupted: during this time the aircraft cannot be assigned to any mission. After a check, an aircraft restores its remaining calendar time and remaining flight time to their maximum values $E^{M}$ and $H^{M}$ respectively. After undergoing a check, there is a minimum amount of periods $E^{m}$ where an aircraft cannot undergo another check.

Some aircraft can be in maintenance at the beginning of the planning horizon, $N_{t}$ is the number of aircraft in such scheduled maintenance per period $t$ and defined only for the first $m - 1$ time periods. Other aircraft are assigned to missions at the beginning of the planning horizon, $\mathcal{A}_{j}^{\text{init}}$ is used to identify aircraft assigned to such missions in the first period.

As in previous work done by \cite{Verhoeff2015}, we define the serviceability of an aircraft as whether it is able or not at a certain moment of time to perform a mission (i.e. is not undergoing maintenance); and we define the sustainability of an aircraft as whether it is able to continue doing missions in the future (i.e. has enough remaining flight time). Finally, we define availability as the total number of periods for which an aircraft is serviceable. All serviceable aircraft have a minimum consumption per period equal to $U^{\min}$ i.e. if they are not assigned to any mission and not in maintenance.

To guarantee both serviceability and sustainability at each time period, aircraft are grouped into clusters. Each cluster represents a subset of aircraft that can do the same types of missions. For each cluster, a minimal number of serviceable aircraft and a minimal number of total remaining flight hours (sustainability) is set as a constraint for each period.

Finally, the objective is to maximize the total availability of the fleet and maximize the final state (i.e. the overall sustainability at the last period). In order to do this, the total number of checks should be minimized and the remaining flight time of the fleet at the last period should be maximized. In the next section, we position our problem in light of the existing contributions in the literature.

\section{State of the Art}
\label{sec:state}

Here we will use aircraft in the most general sense to refer both to airplanes and helicopters. As mentioned before, the planning horizon for maintenance operations may be short-, medium- and long-term. The short term has a time horizon of at most 1 year and is usually divided into periods of one day \cite{marlow2017optimal,Cho2011,Vojvodic2010}. Medium term planning is concerned with a weekly or monthly schedule over 6 months to 2 years \cite{Seif2018,Verhoeff2015,Kozanidis2008, Hahn2008,Pippin1998}. Here, maintenance operations are assigned every 200 to 400 flight hours, which correspond to type A, B and C checks. The capacity for these maintenance types is seen as the number of available man-hours at each period of time. An efficient solving method for a particular case (maximizing overall sustainability) of this problem was presented by \cite{gavranis2015exact} and generalized in \cite{Seif2018}. In \cite{Gavranis2017}, this same technique expanded to deal with the multi-objective version of the problem where overall sustainability is maximized at the same time as its variability is minimized.

Long term planning covers time horizons between five and ten years and mostly addresses scheduling of D checks. These visits are particular in that they last several months. They are scheduled every 1000 - 1200 flight hours or at most 5 years after the last overhaul maintenance. The capacity for this type of maintenance is usually defined by the number of aircraft that can be in maintenance at any given time. 
The first mathematical model for the military Flight and Maintenance Planning problem with respect to long term planning was first presented by \cite{sgaslik1994planning}.
\cite{sgaslik1994planning} used a two-model setup in which a long term (12 monthly periods) problem was solved in order to determine D-type checks and flight hour assignments. Later, this solution was fed into a short-term model where planned missions were assigned to each aircraft in a heterogeneous fleet.



It should be noted that all existing problem formulations are based on the hypothesis that flight hours are continuously assigned to aircraft in order to accomplish a target sum of flight hours (per period or global). Therefore, missions having different flight hours assignments are not explicitly considered.

Another common hypothesis considers an homogeneous fleet, i.e. each aircraft being capable of performing any of the existing missions. The most recent contribution of (\cite{Seif2018}) is an exception with an heterogeneous fleet.

The maintenance capacity is usually defined as the number of working hours available per period. Each maintenance operation is defined by the number of hours required to perform it. 

One can easily imagine that military operations are subject to uncertainty regarding missions, destinations and flight hours. However, contributions incorporating uncertain parameters are quite rare. One of the first attempts to take into account the stochastic nature of maintenance requirements and durations was presented by \cite{Mattila2008}, where a simulation model was built in order to find good maintenance policies. Recently, \cite{Kessler2013} developed a model based on a multi-armed bandit superprocess to choose between two different heuristics or policies in order to maximize the availability of the fleet.

We suggest the following classification of existing FMP formulations. It includes different types of objective functions and constraints.

\subsection{Maintenance related features (\textbf{M})}

  \begin{enumerate}
    \item \textbf{Flight potential:} maximal number of flight hours before a mandatory maintenance.
    \item \textbf{Calendar potential:} maximal number of calendar periods before a mandatory maintenance.
    \item \textbf{Duration:} set number of periods during which an aircraft is immobilized.
    \item \textbf{Capacity:} maximum number of aircraft undergoing maintenance in any given period.
    \item \textbf{Heterogeneous:} maintenance operations may differ for particular types of aircraft, each having its own capacity. An aircraft may have more than one type of maintenance.
    \item \textbf{Flexible:} the overall maintenance capacity is measured in man-hours, a different number of resources can be dedicated to each aircraft at different time periods.
  \end{enumerate}

\subsection{Mission related features (\textbf{V})}

  \begin{enumerate}
    \item \textbf{Min aircraft:} each mission needs a certain number of aircraft of each type in each period.
    \item \textbf{Heterogeneous:} aircraft may not be compatible with all missions.
    \item \textbf{Hour consumption:} a fixed amount of flight hours is required in each period.
    \item \textbf{Min duration:} if an aircraft is assigned to a mission, there's a minimum amount of time it has to be used for this mission and cannot be assigned to another one.
    \item \textbf{Total hours:} the total number of flight hours in the horizon needs to fall within a given range. Sometimes, each group of aircraft has its own range.
    \item \textbf{Min usage:} aircraft that are not assigned to a mission or a maintenance operation are used for a default number of hours per period.
  \end{enumerate}

\subsubsection{Aircraft related features (\textbf{F})}

  \begin{enumerate}
    \item \textbf{Initial state}: takes into account the current status of the aircraft at the beginning of the planning horizon.
    \item \textbf{Sustainability cluster}: minimum amount of flight hours per cluster and period.
    \item \textbf{Serviceability cluster}: minimum amount of serviceable aircraft per cluster and period.
    \item \textbf{Availability:} maximum total amount of maintenance operations.
    \item \textbf{Sustainability:} limit the amount of remaining flight hours:
    \begin{enumerate}
      \item at the last period of the planning horizon, lower bound.
      \item at some periods, lower bound.
      \item at every period during the planning horizon, lower bound.
      \item at the period with the lowest number of hours, maximize the minimum.
      \item at all periods, maximize the sum.
      \item at all periods, minimize the variance.
    \end{enumerate}
  \end{enumerate}

Table \ref{tab:instances} shows the problem formulations used in the previous work and available information about solved instances of the problem. 'Flight hours' are the maximum number of flight hours between 2 maintenance operations. They provide an idea of the type of maintenance taken into account.

\begin{small}
\begin{longtable}{lllllll}
  \toprule
  Parameter &Gavranis          &Cho            &Verhoeff          &Marlow            &Seif     &This \tabularnewline
            &\cite{Gavranis2017} &\cite{Cho2011} &\cite{Verhoeff2015} &\cite{marlow2017optimal} &\cite{Seif2018} & paper\tabularnewline
  \midrule
  \endhead
  Aircraft       &	 50-100  &	 15    &20\cite{J17} &	12-24 &	100 &	 15-60   \tabularnewline
  Periods        &	 6       &	 520   &	 52        &	30    &	 6    &	 90    \tabularnewline
  Period unit    &	 month   & 0.5 day &	 week      &	day   & month &	month  \tabularnewline
  Flying hours   &	 300     &	 300   &   400       &	200   &125-500&800-1200\tabularnewline
  Constraints    &	         &	 30k   &	           &	36k   &	      &	4k-17k \tabularnewline
  Variables      &	         &	 350k  &	           &	20k   &	      &	6k-24k \tabularnewline
  Instances      &	 30      &	 60    &	   2       &	16    &	 30   &	50     \tabularnewline
  Technique      &MIP        &   MIP   &     MIP     &   MIP  &MIP    &MIP     \tabularnewline
\bottomrule
  \caption{Previous work: solved instances }
  \label{tab:instances}\\
\end{longtable}
\end{small}

Table \ref{tab:constraints} shows the constraints taken into account in existing problem formulations.

\begin{small}
\begin{longtable}{lllllll}
\toprule
  Constraints &Gavranis          &Cho            &Verhoeff          &Marlow            &Seif     &This \tabularnewline
            &\cite{Gavranis2017} &\cite{Cho2011} &\cite{Verhoeff2015} &\cite{marlow2017optimal} &\cite{Seif2018} & paper\tabularnewline
\midrule
\endhead
\textbf{M1} Flight potential      & C       & C  & C  & O  & C  & C  \tabularnewline
\textbf{M2} Calendar potential    &         &    &    &    &    & C  \tabularnewline
\textbf{M3} Duration              &         &    &    &    &    & C  \tabularnewline
\textbf{M4} Capacity              & C       &  O & C  & C  &  C & C  \tabularnewline
\textbf{M5} Heterogeneous         &         &    &    &    & C  &    \tabularnewline
\textbf{M6} Flexible              & C       &    & C  & C  & C  &    \tabularnewline
\textbf{V1} Min aircraft          &         &    &    &    &    & C  \tabularnewline
\textbf{V2} Heterogeneous         &         &    &    &    & C  & C  \tabularnewline
\textbf{V3} Hour consumption      & C       & C  & C  & C  & C  & C  \tabularnewline
\textbf{V4} Min duration          &         &    &    &    &    & C  \tabularnewline
\textbf{V5} Total hours           & C       &    & C  & O  &    &    \tabularnewline
\textbf{V6} Min usage             &         &    &    &    &    & C  \tabularnewline
\textbf{F1} Initial states        & C       & C  & C  & C  & C  & C  \tabularnewline
\textbf{F2} Sustainability cluster&         &    &    &    &    & C  \tabularnewline
\textbf{F3} Serviceability cluster&         &    &    &    &    & C  \tabularnewline
\textbf{F4} Availability          &         &    & C  &    &    & O  \tabularnewline
\textbf{F5} Sustainability        &O(e),O(f)&C(a)&O(d)&O(b)&O(e)&O(a)\tabularnewline

\bottomrule
  \caption{Constraints taken into account in the existing formulations: "O" means an objective or a soft constraint, "C" means a hard constraint.}
  \label{tab:constraints}\\
\end{longtable}
\end{small}

As can be seen from this comparison, the majority of existing formulations were developed for an homogeneous fleet that needs to comply with general flight-hour demands under flight hours constraints to control maintenance operations and maintenance capacity constraints (flexible or not).

The formulation developed in this paper includes a series of new constraints that have not been studied before. We introduce a calendar constraint on the frequency of maintenance operations: it now depends not only on the flight hours of an aircraft but also on the calendar time since the last maintenance. The assignment to existing missions is considered explicitly: each mission requires a number of aircraft of a specific type. Each aircraft assignment has a minimal number of periods to serve the same mission. The objective is to maximize the overall availability of the fleet while guaranteeing acceptable levels of sustainability and serviceability per type of aircraft.

\section{Complexity analysis}
\label{sec:complex}

In order to prove the NP-Hardness of the problem considered, we use its reduction to solve the fixed interval scheduling problem introduced below.

\subsection{Fixed Interval Scheduling Problem}

    The ``Fixed Interval Scheduling Problem'' is an NP-Complete problem that was presented in \cite{Krishnamoorthy2001}. It is also referred to as ``shift minimization personnel task scheduling problem'' when related to workers.

    A description paraphrased from \cite{Smet2015} follows:

    Let $P = 1,...,n$ be the set of tasks to be assigned and $E = 1,...,m$ the set of employees. Each task $p \in P$ has a duration $u_{p}$ , a start time $s_{p}$ and an end time $f_{p}$ = $s_{p} + u_{p}$ . Each employee $e$ has a set of tasks $P_{e} \in P$ that he/she can perform. Similarly, for each task $p$, a set $E_{p} \in E$ exists, which contains all employees that can perform task $p$. Both $T_{e}$ and $E_{p}$ are defined based on qualifications, time windows of tasks and the availability of employees. The objective is to minimize the required number of employees needed to perform all tasks.

    This problem uses the following notations.

    $i \in \mathcal{P}_{ref}$:

    \begin{tabular}{ll}
        $\mathcal{E}$       & employees. \\
        $\mathcal{P}$       & tasks.  \\
        $u_p$     			& duration of task $p$. \\
        $s_p$     			& start time of task $p$. \\
        $f_p$     			& end time of task $p$. \\
        $S_e$               & set of tasks  employee $e$ can perform. \\
        $E_p$               & set of employees that can perform task $p$. \\
    \end{tabular}





\subsection{Reduction}

    Now, a simplified description of our original problem is drafted so that it complies with the formulation of the ``Fixed Interval Scheduling Problem''.

	\begin{theorem}
    Finding a feasible solution to the Flight and Maintenance Planning Problem is equivalent to solving the Fixed Interval Scheduling Problem.
	\end{theorem}

	\begin{proof}
    For each employee $e \in E$, we create an analogous aircraft $i \in I$. We will use $e$ and $i$ indistinctly.

    For each task $p \in P$, we create a $j \in J$ mission in our problem with minimal assignment time equal to the duration $A_{j}^{min} = u_{j}$. We will use $p$ and $j$ indistinctly. 
	Start times and end times define the moment when the mission is active: $\mathcal{T}_{j} = t \in \{ s_{j}...f_{j}\}$. The aircraft need for each mission will be constant $R_{j} = 1$. The relationship between tasks and employees is also equivalent: $\mathcal{O}_{i} = S_{i}$ and $\mathcal{I}_{j} = E_{j}$.

    The flight hours for each aircraft on any mission will be $H_{j} = 0$ and the remaining flight hours after a maintenance operation is $H^{M} = 0$. 
	Scheduling of maintenance operations is not considered due to the following setting: $E^{m} = E^{M} = M = N_{t} = 0$ and $Rct_{i}^{Init} = |\mathcal{T}| + 1$.

    Thus, the objective function is reduced to a constant, implying the problem becomes a feasibility problem instead of an optimization problem.

Let:\\
$Q_{ref}$: for an instance $i \in \mathcal{P}_{ref}$: $\exists$ a solution with total resources $\leq k$? \\
$Q_{FMP}$: for an instance $I(i \in \mathcal{P}_{ref})$: $\exists$ a solution with total resources $\leq k$?

Thus, the feasibility version of our problem is NP-complete and its optimization version is NP-hard.


\end{proof}

\section{Mathematical formulation}
  \label{sec:model}

The following model provides a tight MIP formulation that solves the Military Flight and Maintenance Problem described in section \ref{sec:problem}. Maintenance operations will be referred as "checks".

\subsection{Input data}
    \subsubsection{Basic sets}

        \begin{tabular}{p{5mm}p{120mm}}
            $\mathcal{I}$     &  aircraft. \\
            $\mathcal{T}$     &  time periods included in the planning horizon. \\
            $\mathcal{J}$     &  missions. \\
            $\mathcal{K}$     &  cluster of aircraft that share the same functionality. \\
        \end{tabular}

    \subsubsection{Parameters}

        \begin{tabular}{p{10mm}p{115mm}}
            $H_j$             & amount of flight time required by mission $j$. \\
            $R_j$             & number of aircraft required by mission $j$. \\
            $MT_j$            & minimum number of consecutive periods an aircraft has to be assigned to mission $j$. \\
            $U^{min}$       & default aircraft flight time if it is not assigned to any mission nor in maintenance.\\
            $M$               & check duration in number of periods. \\
            $C^{max}$         & maximum number of simultaneous checks. \\
            $E^M$         & maximum number of periods between two consecutive checks. \\
            $E^m$         & minimum number of periods between two consecutive checks. \\
            $H^M$               & remaining flight time after a check. \\
            $N_t$             & number of aircraft known to be in maintenance in period $t$.\\
            $N^K_{kt}$         & number of aircraft in cluster $k$ known to be in maintenance in period $t$.\\
            $A^K_{kt}$         & maximum number of aircraft in cluster $k$ that can be simultaneously in maintenance in period $t$.\\
            $H^K_{kt}$         & minimum number of total remaining flight time for cluster $k$ at period $t$.\\
            $Rft^{Init}_i$  & remaining flight time for aircraft $i$ at the start of the planning horizon. \\
            $Rct^{Init}_i$  & remaining calendar time for aircraft $i$ at the start of the planning horizon. \\
        \end{tabular}

    \subsubsection{Parametric sets}

        \begin{tabular}{p{5mm}p{120mm}}
            $\mathcal{T}_j$     &  time periods $t \in \mathcal{T}$ in which mission $j$ is active. \\
            $\mathcal{J}_t $    &  missions $j \in \mathcal{J}$ to be realized in period $t$. \\
            $\mathcal{I}_j$     &  aircraft $i \in \mathcal{I}$ that can be assigned to mission $j$. \\
            $\mathcal{I}_k$     &  aircraft $i \in \mathcal{I}$ that are included in cluster $k$. One aircraft can belong to more than one cluster. \\
            $\mathcal{O}_i$     &  missions $j \in \mathcal{J}$ for which aircraft $i$ can be used. \\
            $\mathcal{A}^{init}_j$  & aircraft $i \in \mathcal{I}$ that have mission $j$ pre-assigned in the previous period to the start of the planning horizon. \\
        \end{tabular}

    \subsubsection{Time-related parametric sets}

        Several intermediate sets have been defined based on the input data in order to simplify constraint formulation.

        \begin{tabular}{p{8mm}p{117mm}}
            $\mathcal{T}^s_t$ &  time periods $t' \in \mathcal{T}$ such that $t' \in \{ \max{\{1, t - M+1\}},  ..., {t}\}$ (figure \ref{fig:gantt_windows}a). \\
            $\mathcal{T}^m_t$ &  time periods $t' \in \mathcal{T}$ such that $t' \in \{ {t}, ..., \min{\{|\mathcal{T}|, t + M + E^m-1\}}\}$ (figure \ref{fig:gantt_windows}a). \\
            $\mathcal{T}^M_t$ &  time periods $t' \leq |\mathcal{T}| - E^M - M$ such that $t' \in \{ t + M + E^m-1 , ...,  t + M + E^M-1 \}$ (figure \ref{fig:gantt_windows}a). \\
            $\mathcal{T}^{m_{ini}}_i$ &  time periods $t \in \mathcal{T}$ such that $t \in \{ 1, ..., \max{\{0, Rct^{Init}_i - E^M + E^m \}}\}$ (figure \ref{fig:gantt_windows}b). \\
            $\mathcal{T}^{M_{ini}}_i$ &  time periods $t \in \mathcal{T}$ such that $t \in \{ \max{\{0, Rct^{Init}_i - E^M + E^m \}} , ...,  Rct^{Init}_i \}$ (figure \ref{fig:gantt_windows}b). \\
            $\mathcal{T}^{MT}_{jt}$ &  time periods $t' \in \mathcal{T}$ such that $t' \in \{ \max{\{1, t - MT_j\}},  ..., {t}\}$ (figure \ref{fig:gantt_windows}c). \\
        \end{tabular}

    \begin{figure}
        \centering
        \includegraphics[width=\linewidth]{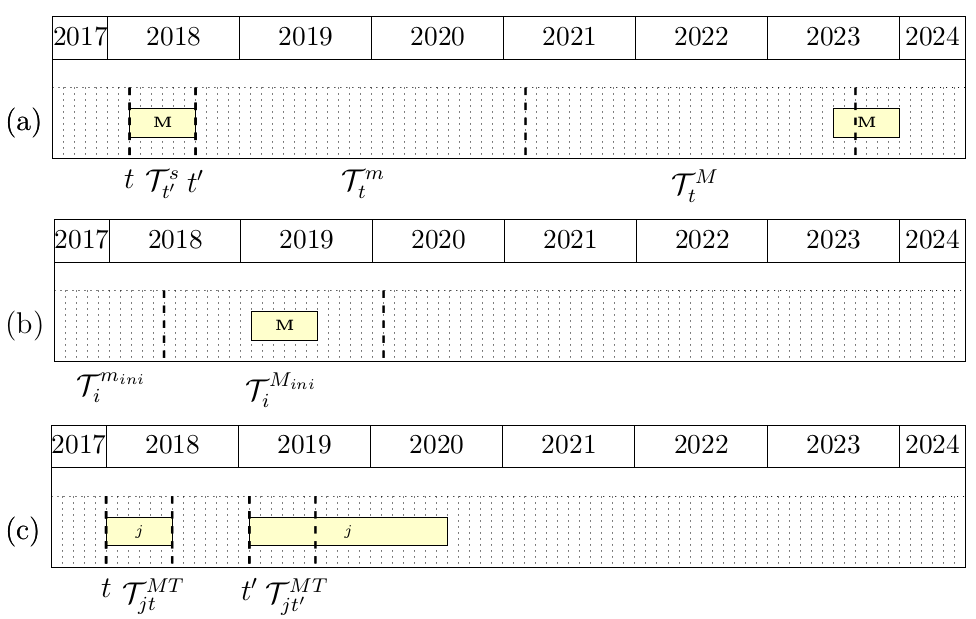}
        \caption{Example showing the maintenance-related time-parametric sets for aircraft $i$. (a) $\mathcal{T}_{t'}^{s}$ refers to the previous $M$ periods to period $t'$. $\mathcal{T}_{t}^{m}$ refers to the periods where a check cannot be planned after starting a check in period $t$. Finally, $\mathcal{T}_{t}^{M}$ refers to the periods where a check needs to be scheduled, after starting a check in period $t$. (b) $Rct_{i}^{Init}$ is equal to 40, meaning a check has to be planned between months number 10 and 39. (c) $\mathcal{T}_{jt}^{MT}$ refers to the periods where the assignment of mission $j$ needs to be kept. In this case, the size is 6.} 
        \label{fig:gantt_windows}
    \end{figure}

\subsection{Variables}

    The following decision variables control the assignment of missions and checks to aircraft.

    \begin{tabular}{p{8mm}p{117mm}}
        $a_{jti}$   &  =1 if mission $j \in J$ in period $t \in \mathcal{T}_j$ is realized with aircraft $i \in \mathcal{I}_j$, 0 otherwise. \\  
        $a^s_{jti}$ &  =1 if aircraft $i$ starts a new assignment to mission $j$ in period $t$. If $a_{jti} = 1$ and $a_{j(t-1)i} = 0$. \\  
        $m_{it}$    &  =1 if aircraft $i \in I$ starts a check in period $t \in \mathcal{T}$, 0 otherwise. \\
    \end{tabular}

    The following decision variables control the used and remaining flight time in aircraft.

    \begin{tabular}{p{8mm}p{117mm}}
        $u_{it}$    &  flown time (continuous) by aircraft $i \in I$ during period $t \in \mathcal{T}$. \\        
        $rft_{it}$  &  remaining flight time (continuous) for aircraft $i \in I$ at the end of period $t \in \mathcal{T}$. \\  
    \end{tabular}

    \paragraph{Fixed values}

    Note that $a_{jti}$ and $m_{it}$ are initially set up to 0 for all aircraft already in maintenance at the beginning of the planning horizon for the remaining time periods of the check. $N_{t}$ is calculated based on this information.
Similarly, for aircraft that have not yet complied with their minimum mission assignment duration at the beginning of the planning horizon, $a_{jti}$ is fixed to comply with the constraints.

\subsection{Objective function and constraints}

    Two objectives have been studied. Objective (\ref{eq:objective1}) minimizes the number of checks. (\ref{eq:objective2}) combines the first one with the goal of maximizing the final total flight hours potential of the fleet.

    \begin{align}
        & \text{Min}\; \sum_{t \in \mathcal{T}, i \in \mathcal{I}} m_{it}  \label{eq:objective1}\\
        & \text{Min}\; \sum_{t \in \mathcal{T}, i \in \mathcal{I}} m_{it} \times H^M - \sum_{i \in \mathcal{I}} rft_{i|\mathcal{T}|} \label{eq:objective2}
    \end{align}
    
    The first term counts all the flight hours given to aircraft following checks and the second term quantifies the amount of remaining flight hours for all aircraft at the end of the planning horizon. These two objectives have the same units, can be easily compared and ensure the aircraft are used in the most efficient way. 

    The following constraints are used in the model:       
    \begin{align}
        & \sum_{t' \in \mathcal{T}^s_t} \sum_{i \in \mathcal{I}} m_{it'} + N_t \leq C^{max}
          & t \in \mathcal{T} \label{eq:capacity1}\\
        & \sum_{i \in \mathcal{I}_j} a_{jti} \geq R_j
                & j \in \mathcal{J}, t \in \mathcal{T}_j  \label{eq:missionres}\\
        & \sum_{t' \in \mathcal{T}^s_t} m_{it'} + \sum_{j \in \mathcal{J}_t \cap \mathcal{O}_i} a_{jti} \leq 1 
                & t \in \mathcal{T}, i \in \mathcal{I} \label{eq:state}
    \end{align}

    Maintenance capacity is controlled by (\ref{eq:capacity1}). The aircraft requirements of missions are defined by (\ref{eq:missionres}). Constraints (\ref{eq:state}) ensure that an aircraft can only be used for one mission or undergo check in the same period.

    \begin{align}
        & a^s_{jti} \geq a_{jti} - a_{j(t-1)i}
                & t =1, ..., \mathcal{T}, j \in \mathcal{J}_t, i \in \mathcal{I}_j \label{eq:start1} \\
        & a^s_{j0i} \geq a_{j0i} - 1 \!1_{i \in \mathcal{A}^{init}_j}
                & j \in \mathcal{J}_0, i \in \mathcal{I}_j \label{eq:start2} \\        
        & \sum_{t' \in \mathcal{T}^{MT}_{jt}} a^s_{jt'i} \leq a_{jti} 
        & j \in \mathcal{J}, t \in \mathcal{T}_j, i \in \mathcal{I}_j \label{eq:start3}
    \end{align}

    Constraints (\ref{eq:start1}) captures period $t$ where aircraft $i$ is firstly assigned to mission $j$ i.e. it was not assigned to it in period ($t - 1$). Constraints (\ref{eq:start2}) are introduced for the first period in the planning horizon.
Constraints (\ref{eq:start3}) control the minimum duration of a consecutive mission assignment. If aircraft $i$ was firstly assigned to mission $j$ in period $t$, it has to be assigned to it during the following $t' \in \mathcal{T}_{jt}^{MT}$ periods. This is a stronger version of the constraint $a_{jt'i}^{s} \leq a_{jti}$.

    To our knowledge, these constraints have not been taken into account in previous military FMP problems.

    \begin{align}
       & \sum_{t' \in \mathcal{T}^s_t} \sum_{i \in \mathcal{I}_k} m_{it'} + N^K_{kt} \leq A^K_{kt}
        &k \in \mathcal{K}, t \in \mathcal{T} \label{eq:serviceability-cluster} \\
       & \sum_{i \in \mathcal{I}_k} rft_{it} \geq H^K_{kt}
        &k \in \mathcal{K}, t \in \mathcal{T} \label{eq:sustainability-cluster}
    \end{align}

    Constraints (\ref{eq:serviceability-cluster}) guarantee a minimum serviceability of aircraft for each cluster $k$. A cluster is defined by the largest group of aircraft that is required exclusively for at least one mission. Constraints (\ref{eq:sustainability-cluster}) ensure there is a minimum amount of remaining flight time for each cluster $k$.

    \begin{align}
         & u_{it} \geq \sum_{j \in \mathcal{J}_t \cap \mathcal{O}_i} a_{jti} H_j 
            & t =1, ..., \mathcal{T}, i \in \mathcal{I} \label{eq:flight1}\\
         & u_{it} \geq U^{min} (1 - \sum_{t' \in \mathcal{T}^s_t} m_{it'})
            & t =1, ..., \mathcal{T}, i \in \mathcal{I} \label{eq:flight2}\\
         & u_{it} \in [0, \max_j{\{H_j\}}]
            & t =1, ..., \mathcal{T}, i \in \mathcal{I} \label{eq:flight_lower}\\
         & rft_{it} \leq rft_{i(t-1)} + H^M m_{it} - u_{it}
            & t =1, ..., \mathcal{T}, i \in \mathcal{I} \label{eq:rft_upper}\\
        & rft_{i0} = Rft^{Init}_i
               & i \in \mathcal{I} \label{eq:rft_initial}\\
        & rft_{it} \geq H^M m_{it'}
                & t \in \mathcal{T}, t' \in \mathcal{T}^s_t, i \in \mathcal{I}\label{eq:rft_lower}\\ 
        & rft_{it} \in [0,H^M]
                & t \in \mathcal{T}, i \in \mathcal{I} \label{eq:mu}
    \end{align}

    The flight time per aircraft and period is calculated in (\ref{eq:flight1})-(\ref{eq:flight_lower}). The remaining flight time is defined by (\ref{eq:rft_upper})-(\ref{eq:rft_initial}) and its limits by (\ref{eq:rft_lower})-(\ref{eq:mu}).

    \begin{align}
        & m_{it'} + m_{it} \leq 1
          & t \in \mathcal{T}, t' \in \mathcal{T}^m_t, i \in \mathcal{I}\label{eq:rct_min}\\ 
        & \sum_{t' \in \mathcal{T}^M_t} m_{it'} \geq  m_{it}
          & t \in \mathcal{T}, i \in \mathcal{I}\label{eq:rct_max}\\
        & m_{it} = 0
          & t \in \mathcal{T}^{m_{ini}}_i, i \in \mathcal{I}\label{eq:rct_min_init} \\
        & \sum_{t \in \mathcal{T}^{M_{ini}}_i} m_{it} \geq  1 
          & i \in \mathcal{I}\label{eq:rct_max_init}
    \end{align}

    The minimum and maximum calendar times are defined by (\ref{eq:rct_min}) and (\ref{eq:rct_max}) respectively. Constraints (\ref{eq:rct_min}) guarantee that if a check is done in some period $t$, we know that another one cannot be done in the immediately consecutive $t' \in \mathcal{T}_{t}^{m}$ periods. Constraints (\ref{eq:rct_max}) ensure that if a check is planned in period $t$, we need to start at least one check in periods $t' \in \mathcal{T}_{t}^{M}$. Constraints (\ref{eq:rct_min_init}) and (\ref{eq:rct_max_init}) control the minimum and maximum remaining calendar times respectively at the beginning of the planning period. They follow the same logic as constraints (\ref{eq:rct_min}) and (\ref{eq:rct_max}), respectively.

    To our knowledge, these constraints have not been taken into account in previous military FMP problems.

\section{Computational experiments}
\label{sec:experim}

The data set for the numerical experiment was generated on the basis of possible data structures used by Air Forces. The main 2 parameters used to generate instances were: the length of the planning horizon and the number of active missions per period. The problem instances were generated in the following way.

\subsection{Notations}\label{nomenclature}

	A discrete choice between optional values is indicated by values separated by commas. Intervals indicate that a value was chosen in a ``uniform random way'' from the intervals for continuous values or through random sampling with replacement for integer values. 
	Values with a * are deterministic control parameters.

\subsection{Sets}\label{sets}

	\begin{small}
	\begin{longtable}[]{lll}
		\toprule
		Code 			& Parameter 							& Value 										\tabularnewline
		\midrule
		\endhead
		\(|J^P|\) 	& Total number of parallel missions* 		& 1, 2, 3				 						\tabularnewline
		\(| I |\) 	& Number of aircraft* 					& 10, 30, 50, 130, 150, 200						\tabularnewline
		\(| T |\) 	& Number of periods* 					& 60, 90, 120, 180 								\tabularnewline
		\(C^{perc}\)& Maintenance capacity (percentage)*	& 0.10, 0.15, 0.2 								\tabularnewline
		\(C^{max}\) & Maintenance capacity 					& \(\lceil C^{perc} \times | I | \rceil\) \tabularnewline
		\bottomrule
	\end{longtable}
	\end{small}

\subsection{Maintenances}\label{maintenances}

	\begin{small}
	\begin{longtable}[]{lll}
		\toprule
		Code & Parameter & Value\tabularnewline
		\midrule
		\endhead
		\(E^M\) & Time limit in periods* & 40, 60, 80\tabularnewline
		\(E^s\) & Time limit window* & 20, 30, 40\tabularnewline
		\(H^M\) & Flight hours limit* & 800, 1000, 1200\tabularnewline
		\(E^m\) & Time limit in periods & \((E^M - E^s)\)\tabularnewline
		\(M\) 	& Check duration* & 4, 6, 8\tabularnewline
		\bottomrule
	\end{longtable}
	\end{small}

\subsection{Missions and flight hours}\label{missions}

	\begin{small}
	\begin{longtable}[]{lll}
		\toprule
		Code & Parameter & Value\tabularnewline
		\midrule
		\endhead
		\(| {T}_j |\) 		& Duration (periods)				& 6 -- 12				\tabularnewline
		\(MT_j\) 			& Minimum assignment (periods) 		& 2, 3, 6				\tabularnewline
		\(R_j\) 			& Number of required aircraft 		& 2 -- 5	 			\tabularnewline
		\(H_j\) 			& Number of required hours 			& triangular(30, 50, 80)\tabularnewline
		\(U^{min}\) 	 	& Default assignment flight hours* 	&0, 5, 15, 20 			\tabularnewline
		\(Y_j\) 			& Type 					   			& choice 1				\tabularnewline
		\(Q_j\) 			& Standard 				   			& 10\% chance			\tabularnewline
		\bottomrule
	\end{longtable}
	\end{small}

	The flight hours are generated using a triangular distribution between 30 and 80 with a mode of 50 and rounded down to the closest integer value. Regarding types and standards, see section \ref{mission-aircraft-compatibility}.



\subsection{Missions duration and start}\label{missions-durations-and-start}

	The following logic has been used in creating missions, assuming $N=|J^P|$ parallel missions at any given moment:

	\begin{enumerate}
		 \item At the beginning we create $N$ missions with a random duration.
		 \item Every time a mission ends, we create a new mission with new random parameters.
		 \item When the end of the planning horizon is reached: the last mission is truncated.
	\end{enumerate}

	This guarantees that there are \emph{always} $N$ missions in parallel at any given time.

\subsection{Aircraft}\label{aircraft}
	
	Each aircraft has specific characteristics that allow it to accomplish missions. These characteristics are represented by a type and a standard. More detail on types and standards is discussed in \ref{mission-aircraft-compatibility}.

	\begin{small}
	\begin{longtable}[]{lll}
		\toprule
		Code & Parameter & Value\tabularnewline
		\midrule
		\endhead
		\(Y_i\) & Type & choice\tabularnewline
		\(Q_i\) & Standards & choice\tabularnewline
		\bottomrule
	\end{longtable}
	\end{small}

\subsection{Aircraft's initial state}\label{aircraft-initial-state}

	\begin{small}
	\begin{longtable}[]{lll}
		\toprule
		Code 				& Parameter 									& Value 							\tabularnewline
		\midrule
		\endhead
		\(NP\) 				& Percentage of aircraft starting in maintenance.  & \(0 - C^{max}\)					\tabularnewline
		\(NV\) 				& Number of aircraft in maintenance. 				& \(| I | \times NP\)			\tabularnewline
		\(At_j\)			& Number of periods previously done under mission \(j\)& \(0 - 2MT_j\)			 			\tabularnewline
		\(Rct^{Init}_i\) 	& Remaining calendar time 						& \(0 - E^M\)						\tabularnewline
		\(Rct^{I2}_i\) 		& Remaining calendar time + noise 				& \(Rct^{Init}_i\) + {[}-3 -- 3{]}	\tabularnewline
		\(Rft^{Init}_i\) 	& Remaining flight time 							& \(Rct^{I2}_i \frac{H^M}{E^M}\)\tabularnewline
		\(NM_i\) 			& Remaining maintenance periods						& 0 -- \(M\)							\tabularnewline
		\bottomrule
	\end{longtable}
	\end{small}

	The initial states will be simulated according to the following rules. To obtain the aircraft $I$ that are in maintenance: (i) $NV$ aircraft will be taken randomly from the set of aircraft. (ii) for these aircraft, $NM_{i}$ will be generated randomly between 0 and $|M|$.

	For the remaining $I - NV$ aircraft that are not in maintenance: (i) $Rct_{i}^{Init}$ and $Rft_{i}^{Init}$ will be generated in a correlated way. The first one is generated randomly and the second one is adjusted. 
	(ii) for each mission $j$ belonging to the set of missions active at the beginning of the planning period: $R_{j}$ aircraft will be taken and assigned to each such a mission with $At_{j}$ previous assignments.


\subsection{mission-aircraft compatibility}\label{mission-aircraft-compatibility}



	Each mission and aircraft has one and only one type. They have to match for mission execution. One aircraft may have multiple standards. A mission may have at most one standard. If it the case, the standards also have to match between missions and aircraft for mission execution.

	Mission parameters are generated in the following way. For each mission, a type $Y_{j} \in Y$ and a standard $Q_{j} \in Q$ will be assigned. $Q_{j}$ can be null, which implies the mission has no standard. A minimum number of aircraft of each type is calculated based on $\sum_{\{ j \in J \mid Y_{j} = y\}}^{}R_{j}\,,\forall y$.

	In order to guarantee a feasible number of aircraft to comply with missions, the requirements for each type of aircraft are calculated for the whole planning horizon. Then, this number of aircraft of each type is created, at least. For the remaining aircraft, their type is chosen randomly taking the weight of the requirements for each type of aircraft into account. In order to guarantee a feasible number of aircraft per standard, we chose to generate twice the number of required standards among the aircraft.

\subsection{Cluster and service levels}\label{cluster-service-levels}

	A cluster is a group of missions that have exactly the same requirements (i.e. same type and standard). To explain the model input parameters, the following notations are needed: $Q_{k}$ is the number of candidates for cluster $k$ and $QH_{k} = Q_{k} \times H^{M}$ is the maximum flight hours for the whole set of aircraft in a given cluster $k$.

	\begin{small}
	\begin{longtable}[]{lll}
		\toprule
		Code 			& Parameter 										& Value 				\tabularnewline
		\midrule
		\endhead
		$AN^K$	& Minimal number of serviceable aircraft per cluster.*		& 1, 2, 3				\tabularnewline
		$AP^K$	& Percentage of serviceable aircraft per cluster.* 		& 0.05, 0.1, 0.2		\tabularnewline
		$HP^K$ 	& Percentage of sustainability per cluster.*		 		& 0.3, 0.5, 0.7			\tabularnewline
		$H^K_{kt}$& Minimal remaining flight hours for cluster $k$. 			& $HP^K \times QH_k$	\tabularnewline
		$A^K_{kt}$& Minimal serviceable aircraft for cluster $k$. 			& $\max{\{AP^K Q_k, AN^K\}}$	\tabularnewline
		\bottomrule
	\end{longtable}
	\end{small}

\subsection{objective functions}

    Two objectives have been studied. When using the configuration $\max\{ rft\}$ = 1, objective \ref{eq:objective2} is being used. Otherwise, $\max\{ rft\}$ = 0, objective \ref{eq:objective1} is used.

\subsection{Software tools}

	Python with the PuLP library was used to build the models. To solve the models, CPLEX 12.8 was used. 
	All tests were run on a 12-core, 64 GB RAM machine running Linux Fedora 20 with a CPU speed (in MHz) of 2927.000.



\section{Results}
\label{sec:results}

Several scenarios were created using the techniques explained in section \ref{sec:experim}. The base and studied scenarios were built based on the values reported in table \ref{tab:summary-scenario}. Only one parameter was changed at a time.

\subsection{Base scenario}

    For each scenario, 50 instances were randomly generated. Among scenarios, the same position of instance always had the same random seed. This was done so that random differences between instances in the same position among different scenarios would be as small as possible and comparisons could be more broadly generalized.

    An experiment is defined as a group of scenarios. Two experiments were run to analyze the sensitivity to parameters and problem size. One third experiment involved the use of generated feasible solution as input to the model.

    \begin{small}
    \begin{longtable}[]{llll}
    	\toprule
    	Parameter      & Name & Base scenario  & Studied scenarios \tabularnewline
    	\midrule
    	\endhead
        $E^s$          &  maintenance calendar time size     & 30             & 20, 40           \tabularnewline
        $E^M$          &  maintenance calendar time      & 60             & 40, 80           \tabularnewline
        $H^M$          &  maintenance flight hours         & 1000           & 800, 1200        \tabularnewline
        $C^{perc}$     &  capacity in percentage of fleet        & 0.15           & 0.1, 0.2         \tabularnewline
        $M$            &  maintenance duration        & 6              & 4, 8             \tabularnewline
        $| J^P |$      &  number of parallel tasks    & 1              & 2, 3, 4          \tabularnewline
        $| T |$        &  number of periods in horizon            & 60             & 120, 140         \tabularnewline
        $U^{min}$      &  minimum flight hours consumption      & 0              & 5, 15, 20        \tabularnewline
        $HP^K$         &  minimum $rft$ per cluster        & 0.5            & 0.3, 0.7         \tabularnewline
        $\max \{rft\}$ &  maximize $rft$ at the end         &  0             & 1                \tabularnewline
    	\bottomrule
    	\caption{Experiments and studied scenarios. 'Base scenario' corresponds to the default values. 'Studied scenarios' corresponds to the values that were used to create the scenarios.}
    	\label{tab:summary-scenario}
    \end{longtable}
    \end{small}

    \begin{figure}
        \centering
        \includegraphics[width=0.8 \linewidth]{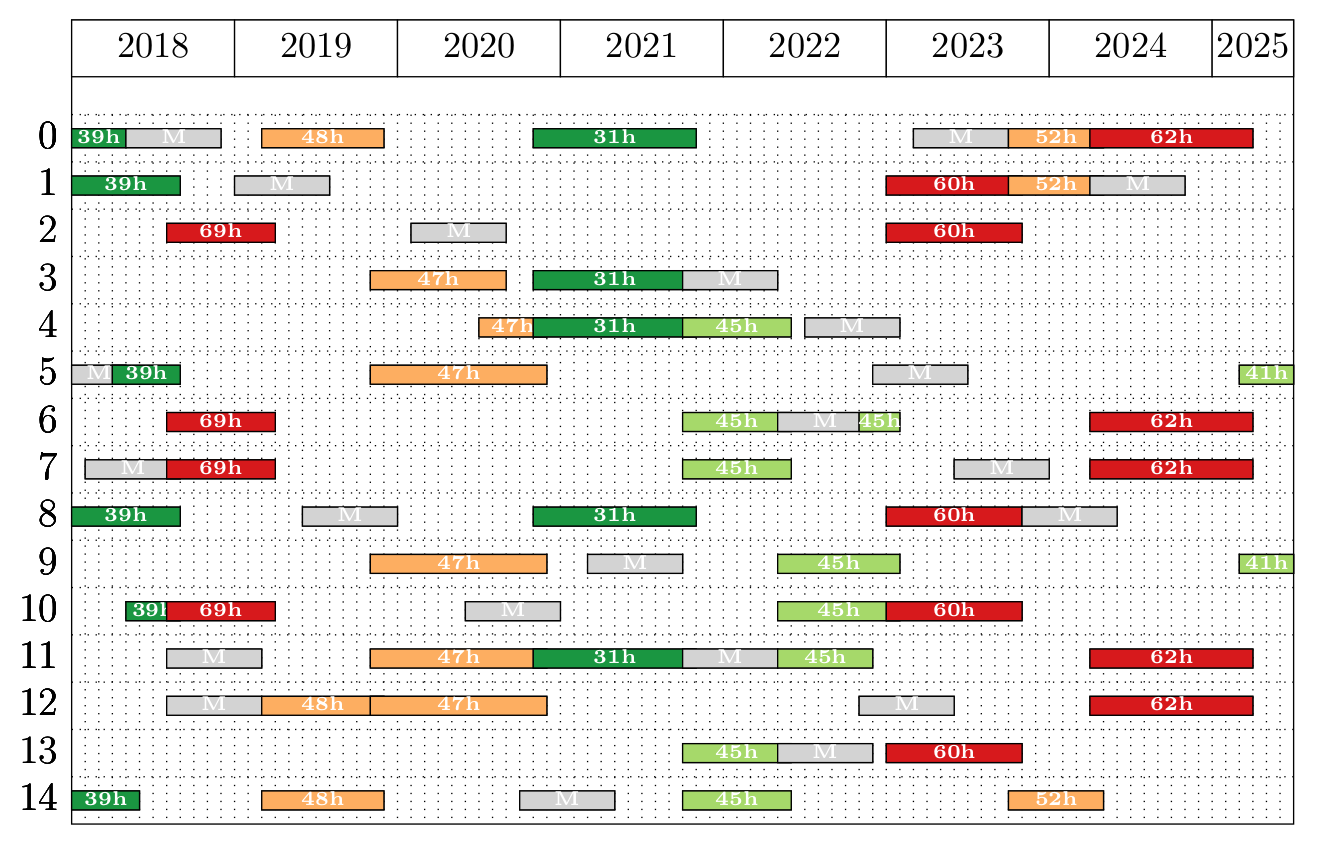}
        \caption{A solution for the base case. Each row is an aircraft, each column represents a month. Missions in red consume the most flight hours, missions in green the least, numbers show the monthly hourly consumption for missions. "M" is used to signal checks. } \label{fig:gantt_example}
    \end{figure}

    An example of a solution for the base case is shown in figure \ref{fig:gantt_example}.

\subsection{Parameter sensibility analysis}

    Experiment 1 consisted in analyzing the sensitivity of the model to changes in its input parameters. Table \ref{tab:experiment1} summarizes the performance after solving the model with each scenario. It can be seen that most instances were solved to optimality, although the resolution times were close to the imposed 1-hour limit. The variations in the size of the problem are due to the differences in the solver's pre-solving capabilities given the fact that these scenarios did not change the size of the original problem.

    \begin{small}
    \begin{table}
        \centering
    	\begin{tabular}{lrrrrrrrr}
\toprule
         case &  $t^{min}$ &  $t^{avg}$ &  non-zero &   vars &   cons &  no-int &  inf &  $g^{avg}$ \\
\midrule
 $HP^{K}$=0.3 &        1.8 &        5.2 &   50976.5 & 4275.5 & 6273.0 &       0 &    0 &        0.0 \\
 $H^{M}$=1200 &        2.1 &       76.7 &   51030.3 & 4295.0 & 6298.2 &       0 &    0 &        0.1 \\
   $E^{s}$=20 &        1.6 &      172.8 &   29772.8 & 3826.1 & 5120.3 &       0 &    3 &        0.2 \\
   $E^{s}$=40 &        4.0 &      266.8 &   64152.6 & 4496.1 & 6994.6 &       0 &    1 &        0.4 \\
         base &        2.2 &      310.6 &   51167.1 & 4310.7 & 6315.9 &       0 &    1 &        0.3 \\
   $E^{M}$=40 &        8.1 &      530.9 &   68612.7 & 4525.5 & 7632.9 &       0 &    0 &        0.2 \\
   $E^{M}$=80 &        1.5 &     1250.6 &   28257.9 & 3877.8 & 5010.4 &       0 &    3 &        1.9 \\
 $HP^{K}$=0.7 &       80.7 &     1746.9 &   50805.8 & 4393.9 & 6320.6 &       0 &   42 &        2.9 \\
  $H^{M}$=800 &        4.4 &     2168.5 &   51219.7 & 4327.2 & 6327.2 &       0 &    5 &        2.7 \\
  $U^{min}$=5 &       24.6 &     2650.3 &   60950.1 & 5525.3 & 8583.6 &       0 &    3 &        4.3 \\
 $U^{min}$=20 &     3600.0 &     3600.0 &   53562.3 & 5379.8 & 8149.6 &      25 &    8 &        5.2 \\
 $U^{min}$=15 &     3600.0 &     3600.0 &   60716.4 & 5529.0 & 8573.6 &      10 &    6 &        6.3 \\
\bottomrule
\end{tabular}
        \caption{Experiment 1 summary per scenario sorted by average solving time. 'vars', 'cons' and 'non-zero' correspond to the average number of variables, constraints and non-zero values in the 50 instances, respectively. The 'no-int' column corresponds to the number of instances in which an integer solution was not found, even though the problem was not considered infeasible. $t^{min}$ and $t^{avg}$ refer to the minimum and average values for the solving times, in seconds. All $t^{\max}$ values are 3600. $g^{avg}$ corresponds to the average gap, in \%.}
        \label{tab:experiment1}
    \end{table}
    \end{small}

    \begin{figure}
        \centering
        \includegraphics[width=0.8 \linewidth]{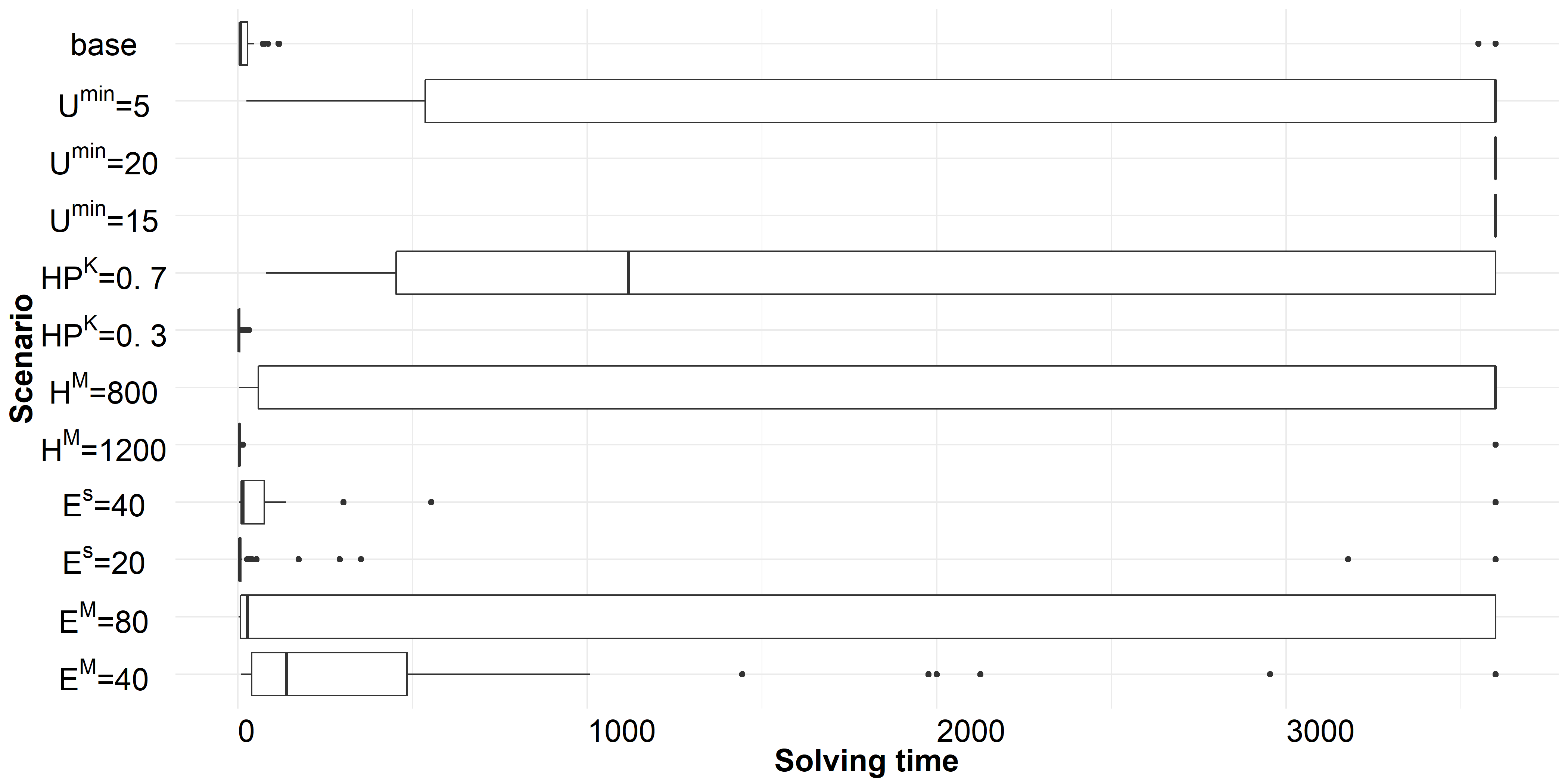}
        \caption{Box-plot showing the distribution of solution times for each of the instances of Experiment 1. Scenarios are shown in the X-axis while the times are shown in the Y-axis.} \label{fig:experiment1-times}
    \end{figure}

    The results obtained show that parameters with influence on execution times and in remaining relative gaps included the ones that regulate the frequency of checks, e.g. the amount of flight hours between checks ($H^{M}$): increasing available hours, without changing the flight load, will dramatically reduce solution times (see figure \ref{fig:experiment1-times}). This modification also has an impact on whether a solution is feasible or not (see table \ref{tab:experiment1}). Another parameter that had a very sensible impact was the minimum amount of sustainability per cluster $HP^{K}$. The impact of both of these parameters can also be confirmed via the difference in the average needed nodes to reach optimality, shown in table \ref{tab:experiment1-relax}.



    \begin{figure}
        \centering
        \includegraphics[width=0.8 \linewidth]{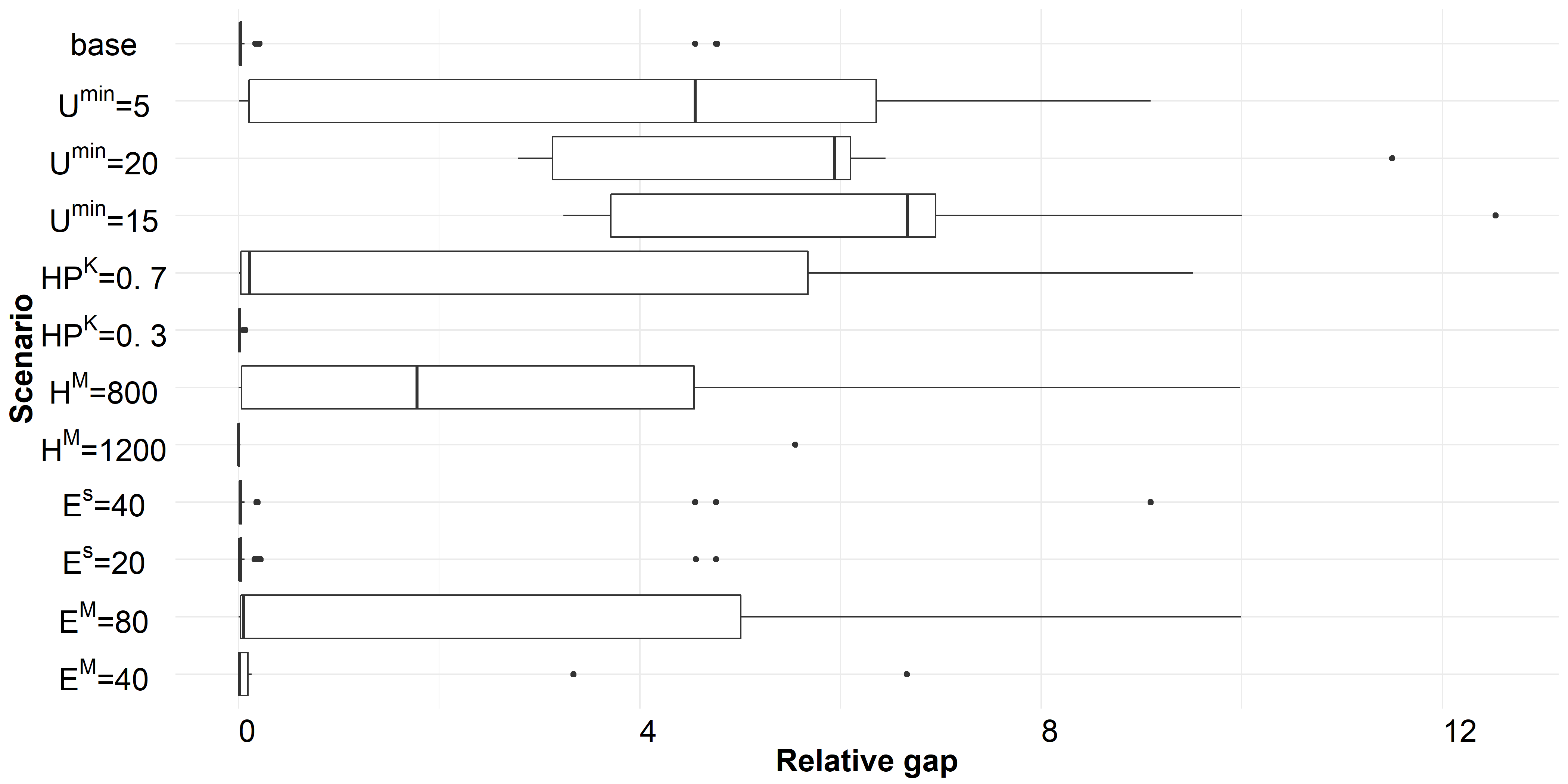}
        \caption{Box-plot showing the distribution of relative gaps for each of the instances of Experiment 1. Scenarios are shown in the X-axis while the gaps are shown in the Y-axis.} \label{fig:experiment1-gaps}
    \end{figure}



    The minimum consumption of flight hours per period $U^{min}$ makes the problem significantly harder to solve. This can be confirmed both via the remaining gaps, solving times and with several instances where a feasible solution was not found after 1 hour (table \ref{tab:experiment1}). The effect was evident even when adding a relatively small quantity of consumption hours ($U^{min}$=5), although a greater impact is correlated with a higher minimum consumption. In addition, table \ref{tab:experiment1} shows that the solver's pre-processor is less able to reduce the problem size in these scenarios than in most of the other ones. Table \ref{tab:experiment1-relax} shows how, although the initial relaxation is particularly bad for these scenarios, the cuts phase (helped by a manual configuration of the solver) significantly improves the relaxation.

    \begin{small}
    \begin{table}
        \centering
        \begin{tabular}{lrrrr}
\toprule
           case &  rinit &  rcuts &  icuts &   nodes \\
\midrule
           base &   15.9 &    1.2 &    8.0 &  4335.9 \\
     $E^{s}$=20 &    2.1 &    0.7 &    2.9 &  2999.0 \\
     $E^{s}$=40 &   17.5 &    2.5 &   23.3 &  4384.3 \\
     $E^{M}$=40 &   41.0 &    5.0 &    7.3 & 15172.6 \\
     $E^{M}$=80 &    4.5 &    2.3 &    7.4 & 46174.8 \\
   $H^{M}$=1200 &   17.1 &    0.1 &    3.5 &   121.2 \\
    $H^{M}$=800 &   10.3 &    5.0 &   17.8 & 89602.2 \\
   $HP^{K}$=0.3 &   16.7 &    0.1 &    3.1 &   372.8 \\
   $HP^{K}$=0.7 &   16.7 &    9.1 &   13.5 & 26534.8 \\
   $U^{min}$=15 &   23.8 &    7.4 &    8.0 &      \\
   $U^{min}$=20 &   20.1 &    6.3 &    3.5 &       \\
    $U^{min}$=5 &   18.2 &    6.8 &   22.3 & 13642.6 \\
 $C^{perc}$=0.2 &   15.1 &    1.0 &    8.8 &  1386.2 \\
\bottomrule
\end{tabular}

        \caption{Experiment 1: mean performance of relaxations per scenario (in \% difference). 'rinit' compares the first continuous relaxation and the best solution found. 'rcuts' compares the continuous relaxation after cuts in the root node and the best solution found. 'icuts' compares the best solution found after cuts in the root node and the best solution found. Finally, 'nodes' measures the nodes in the branch and bound it took to prove optimality in the instances where it was proved.}
        \label{tab:experiment1-relax}
    \end{table}
    \end{small}

\subsection{Problem size sensibility analysis}
    
    Experiment 2 studied variants in the problem size and the objective function. First of all, the horizon was increased in size by changing the amount of planning periods. Second, the number of parallel tasks was increased with an equivalent increase in the size of the fleet. Lastly, an objective function that maximizes the final state in addition to minimizing the number of checks was tested.

    \begin{small}
    \begin{table}
        \begin{tabular}{lrrrrrrrr}
\toprule
             case &  $t^{min}$ &  $t^{avg}$ &  non-zero &    vars &    cons &  no-int &  inf &  $g^{avg}$ \\
\midrule
             base &        2.2 &      310.6 &   51167.1 &  4310.7 &  6315.9 &       0 &    1 &        0.3 \\
      $| T |$=120 &       19.9 &      376.0 &   88668.3 &  5815.8 &  9161.6 &       0 &    6 &        0.8 \\
      $| J^P |$=2 &       20.1 &     1313.9 &  101572.8 &  8317.9 & 12266.2 &       0 &    5 &        0.8 \\
      $| T |$=140 &       44.5 &     1651.0 &  115910.9 &  6738.7 & 11081.5 &       0 &    4 &        3.3 \\
 $\max \{rft\}$=1 &        7.5 &     2198.8 &   51167.1 &  4310.7 &  6315.9 &       0 &    1 &        1.6 \\
      $| J^P |$=3 &       62.7 &     2723.7 &  157213.7 & 12907.6 & 18915.8 &       1 &    8 &        2.5 \\
      $| J^P |$=4 &      114.7 &     3228.6 &  209747.2 & 17045.4 & 24947.4 &       3 &    9 &        2.6 \\
\bottomrule
\end{tabular}

        \caption{Experiment 2 summary per scenario sorted by average solving time. 'vars', 'cons' and 'non-zero' correspond to the average number of variables, constraints and non-zero values in the 50 instances, respectively. The 'no-int' column corresponds to the number of instances in which an integer solution was not found, even though the problem was not considered infeasible. $t^{\min}$ and $t^{avg}$ refer to the minimum and average values for the solving times, in seconds. All $t^{\max}$ values are 3600. $g^{avg}$ corresponds to the average gap, in \%.}
        \label{tab:experiment2}
    \end{table}
    \end{small}

    By activating maximization of the end state for the whole fleet ($\max\{ rft\}$=1), the efficiency of the solving process, measured in solving times, declines significantly. Another condition with a similar effect is increasing the size of the planning horizon ($|T|$=140). Both scenarios seem to share the same difficulty.

    A similar effect was detected when increasing the number of parallel missions $|J^{P}|$ and the size of the fleet proportionally. This effect can be explained by the fact that the model size grows in proportion to the number of parallel missions (see 'non-zero' column in table \ref{tab:experiment2}).

    To sum up, although the model performance seems to deteriorate with larger instances, the effect in resulting gaps seems to keep a lineal relationship with regards to $|J^{P}|$ and $|T|$, for the studied scenarios and the resulting gaps are still acceptable (see \ref{fig:experiment2-gaps}).

    \begin{figure}
        \centering
        \includegraphics[width=0.8 \linewidth]{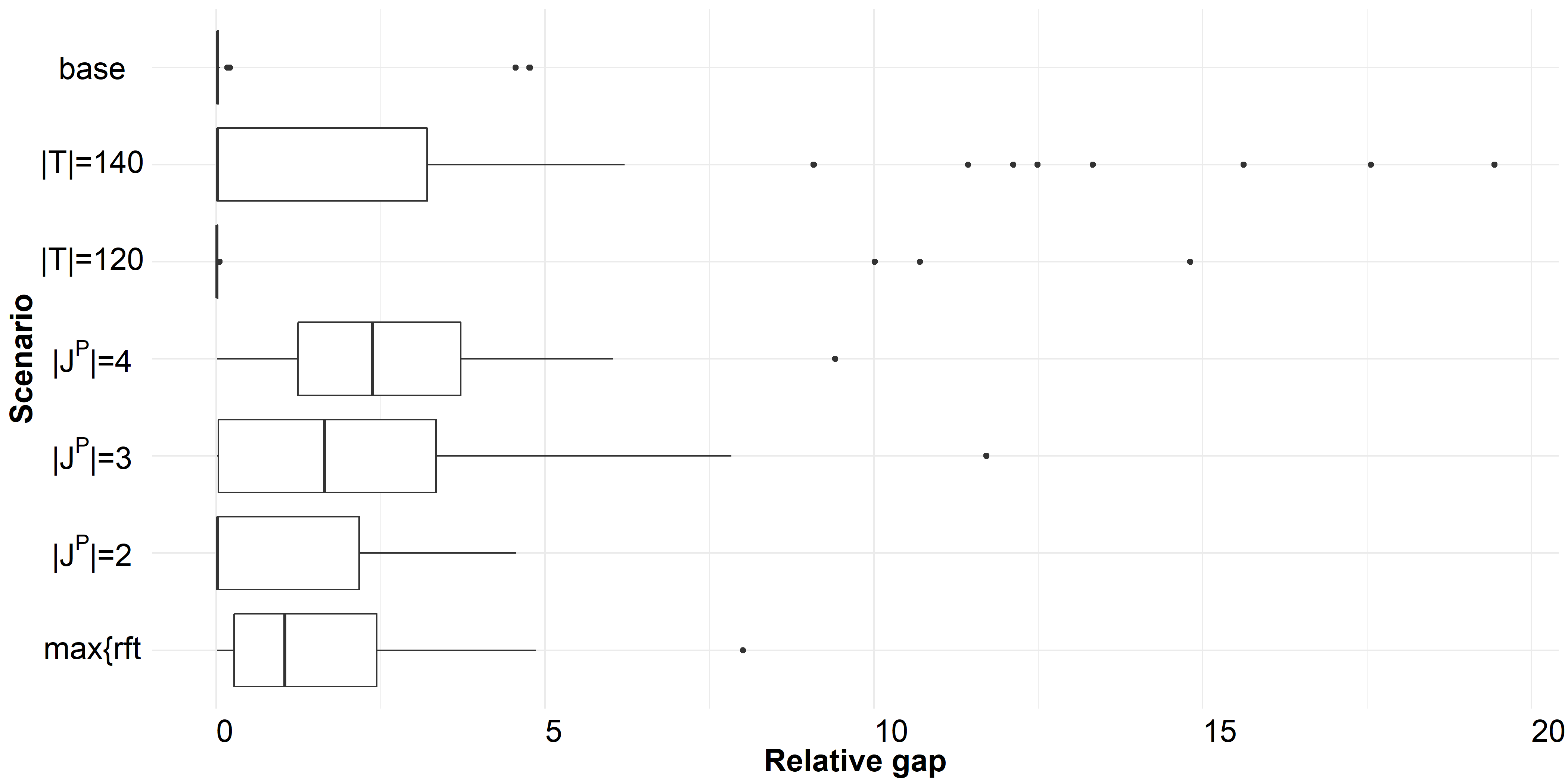}
        \caption{Box-plot showing the distribution of relative gaps for each of the instances of Experiment 2. Scenarios are shown in the X-axis while the gaps are shown in the Y-axis.} \label{fig:experiment2-gaps}
    \end{figure}

    \begin{figure}
        \centering
        \includegraphics[width=0.8 \linewidth]{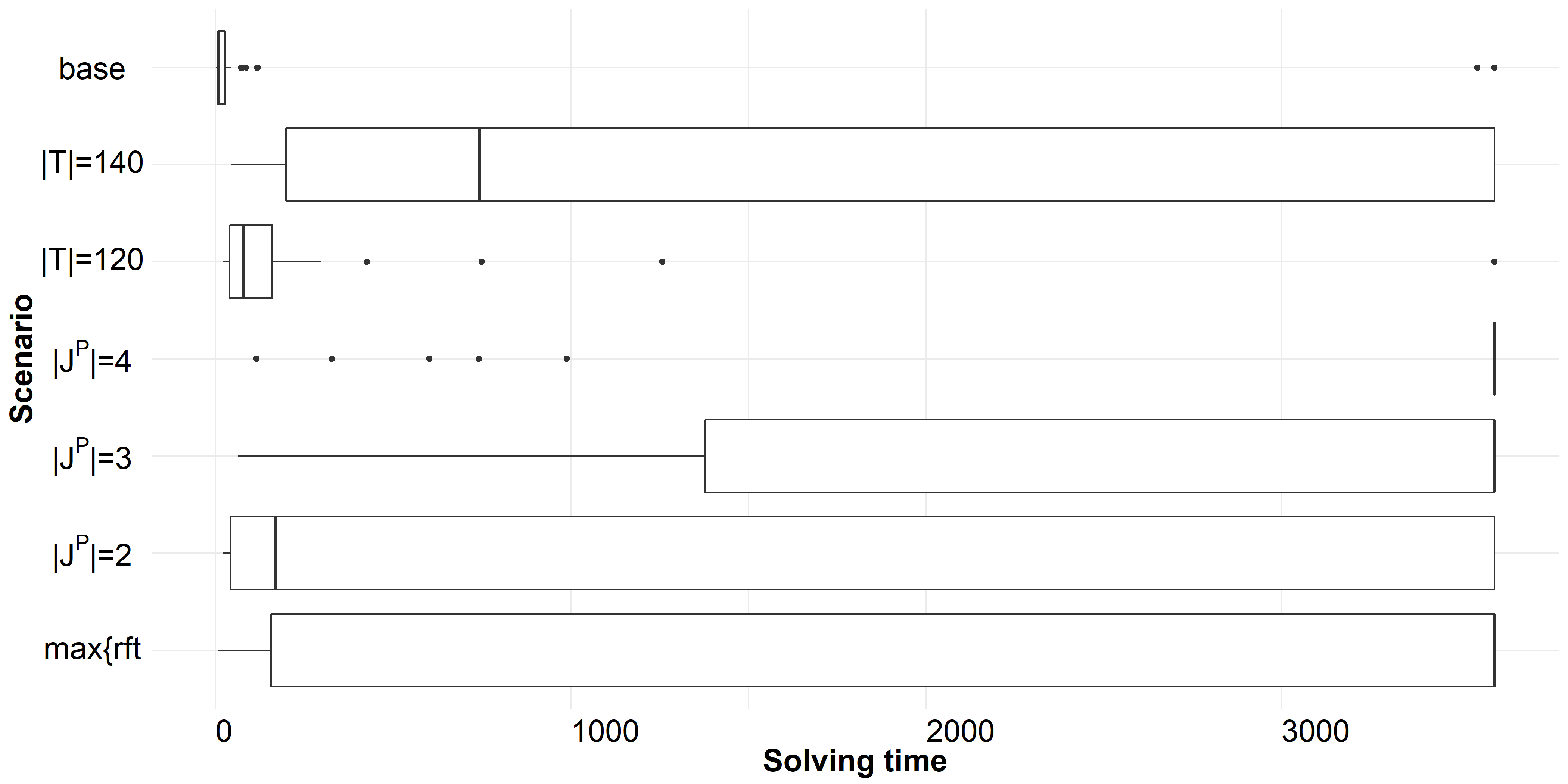}
        \caption{Box-plot showing the distribution of solution times for each of the instances of Experiment 2. Scenarios are shown in the X-axis while the times are shown in the Y-axis.} \label{fig:experiment2-times}
    \end{figure}

    Table \ref{tab:experiment2-relax} how the quality of the cuts phased decreases with the size of the planning horizon, in relaxation quality as in integer solution quality. Also, the number of nodes needed to find an optimal solution considerably increases in the ($| T |$=140) scenario. This is possible due to the fact that aircraft need a third maintenance in these circumstances and the possible maintenance combinations grow in a combinatorial sense. Lastly, guaranteeing an optimal solution appears to prove difficult when considering the final state in the objective function, as seen in the average number of nodes needed.

    \begin{small}
    \begin{table}
        \centering
        \begin{tabular}{lrrrr}
\toprule
             case &  rinit &  rcuts &  icuts &    nodes \\
\midrule
             base &   15.9 &    1.2 &    8.0 &   4335.9 \\
      $| J^P |$=2 &   15.1 &    2.1 &    9.9 &  10518.1 \\
      $| J^P |$=3 &   16.1 &    4.0 &   17.9 &  11876.2 \\
      $| J^P |$=4 &   15.2 &    3.8 &   11.4 &   9875.4 \\
      $| T |$=120 &   24.8 &   10.4 &   21.7 &   5239.8 \\
      $| T |$=140 &   22.4 &   15.6 &   23.7 &  23023.9 \\
 $\max \{rft\}$=1 &    7.2 &    3.3 &   59.3 & 142925.7 \\
\bottomrule
\end{tabular}

        \caption{Experiment 2: mean performance of relaxations per scenario (in \% difference). 'rinit' compares the first continuous relaxation and the best solution found. 'rcuts' compares the continuous relaxation after cuts in the root node and the best solution found. 'icuts' compares the best solution found after cuts in the root node and the best solution found. Finally, 'nodes' measures the nodes in the branch and bound it took to prove optimality in the instances where it was proved.}
        \label{tab:experiment2-relax}
    \end{table}
    \end{small}

\subsection{Heuristic comparison}

    A heuristic was built to generate feasible solutions using a simulated annealing logic where each move consist of a (1) release and (2) repair action. The stop criteria are three: (a) a time limit, (b) getting a feasible solution or (c) an iteration limit. Release actions consist in un-assigning mission and checks for some aircraft-period combinations present in the solution. Repair actions consist on assigning new mission and checks to the solution in order to comply with requirements. Candidate moves are generated using the location of errors in the solution and the move is chosen randomly from those candidates. The move is then accepted or not depending on the temperature in the system and the improvement with respect to the previous solution.

    Experiment 3 studied the impact of using this heuristic to generate fast feasible solutions for instances. These solutions were, firstly, compared to the best available solutions obtained using the mathematical model (usually optimal) and, later, used as input in order to warm-start the solution process by the solver.



    Table \ref{tab:heuristic_comp} shows three ways to properly measure the heuristic's performance: (1) the average time it takes to find an initial solution ($t^{avg}_H$), (2) the distance from that first initial solution to the best known one ($\%Dif_H$) and (3) the probability of finding an initial solution in a short time (10 minutes) ($\%Init_H$). The relative quality seems to depend particularly on the type of objective function being used ($\max \{rft\}$=1 scenario) but not so on the size of the problem. On the other hand, the probability of finding a solution appears to depend on the number of parallel tasks at any given time. Finally, the average times to find a solution do increase with problem size but not in a uncontrollable way, especially for increases in planning horizon size.

    \begin{small}
    \begin{table}
        \centering
        \begin{tabular}{lrrrrrrr}
\toprule
             case &  $t^{avg}_M$ &  $t^{avg}_H$ &  $t^{avg}_{M+H}$ &  $g^{avg}_M$ &  $g^{avg}_{M+H}$ &  $\%Dif_H$ &  $\%Init_H$ \\
\midrule
             base &        247.1 &         23.7 &            183.0 &          0.2 &              0.1 &       22.0 &        95.9 \\
      $U^{min}$=5 &       2417.5 &         92.8 &           2364.2 &          3.8 &              3.4 &       19.0 &        74.5 \\
      $| J^P |$=2 &        706.9 &         86.2 &            976.9 &          0.4 &              0.7 &       22.0 &        64.4 \\
      $| J^P |$=3 &       1979.5 &        217.1 &           1712.3 &          1.6 &              1.0 &       23.0 &        34.1 \\
      $| J^P |$=4 &       3102.1 &        401.1 &           2963.7 &          1.1 &              1.1 &       19.5 &        18.4 \\
      $| T |$=120 &        247.6 &         47.8 &            305.0 &          0.5 &              0.8 &       19.7 &        75.0 \\
      $| T |$=140 &       1493.3 &         65.8 &           1211.7 &          2.9 &              1.9 &       23.4 &        87.0 \\
 $\max \{rft\}$=1 &       2214.1 &         23.7 &           2194.2 &          1.6 &              1.7 &       89.9 &        95.9 \\
\bottomrule
\end{tabular}

        \caption{Comparison of all instances where a feasible solution was found by the heuristic in a selected set of difficult scenarios. $\%Init_H$ shows the percentage of instances where the heuristic found a feasible solution before 10 minutes compared to the total number of solutions found by the MIP model after one hour. $\%Dif_H$ measures the relative distance between the solution found by the heuristic and the best solution found with the MIP model. Three average times are shown for the MIP model's instances ($t^{avg}_M$), the heuristic's time ($t^{avg}_H$) and the MIP model that was fed a feasible solution as starting solution ($t^{avg}_{H+M}$). Optimality gaps $t^{avg}_{M}$, $t^{avg}_{H+M}$ for both models are also shown.}
        \label{tab:heuristic_comp}
    \end{table}
    \end{small}

    Secondly, feeding an initially generated solution to the solver slightly increases the solving process, both in resolution times and in gap, although not in a meaningful quantity. Taking into account the heuristic performance and impact on resolution, it can be concluded that it is particularly useful for longer planning horizons, where the performance remains high and the impact is also greatest.

    Finally, since the solver permitted it, giving a nearly-feasible solution to the solver was tested. Usually this solution is then repaired by the solver during the cutting phase. No gains in solution times and gap were observed for these cases.
\section{Conclusions and further work}
\label{sec:conclusions}

This paper presented a new MIP formulation for the long-term Flight and Maintenance Planning problem for military aircraft. Its performance was measured by solving an array of scenarios inspired by real French Air Force needs.

Compared to the existing literature, the problem studied includes several new constraints while still managing to solve fairly large instances. Also, a complexity proof was presented.

The study showed that the mathematical model's performance is quite robust with respect to increases in fleet size, number of missions and the size of the planning horizon. On the other hand, adding fixed additional consumptions outside of missions proved challenging.

In terms of performance, gains in resolution time were obtained by developing a construction heuristic that provided starting solutions for the cases where an integer solution is not easily obtained by the model. It was shown to be potentially useful in scenarios with long planning horizons.

With respect to extending the model, additional constraints from real world application, such as long-term storage of grounded aircraft, can be incorporated.

In order to better integrate long term schedules with the existing medium- and short-term maintenance planning, a matheuristic that alternates between the two problems could potentially satisfy the needs of the different scopes with a good quality solution that takes several types of aircraft maintenance into account simultaneously.

Regarding uncertainty treatment, explicit ways to measure the stochastic nature of the input parameters can be implemented. For example, by using robust optimization or stochastic programming in order to guarantee the feasibility of the solution even in extreme scenarios.













\end{document}